\documentclass[runningheads]{llncs}
\usepackage[utf8]{inputenc} 
\usepackage[T1]{fontenc}
\usepackage{fullpage}
\usepackage{amsmath}
\usepackage{enumerate}
\usepackage{mathtools}
\usepackage{nicefrac}
\usepackage{appendix}
\usepackage{booktabs}
\usepackage{adjustbox}
\usepackage{hyperref}       
\usepackage{amsfonts}       
\usepackage{microtype}      
\usepackage{thm-restate}
\usepackage{cleveref}
\usepackage[inline]{enumitem}
\usepackage{tikz}
\usepackage{natbib}
\usepackage{siunitx}
\usepackage{booktabs}
\usepackage{float} 
\usepackage{wrapfig}

\usetikzlibrary{shapes}
\usetikzlibrary{positioning}

\usepackage{subcaption}
\usepackage[font=small,labelfont=bf]{caption}

\begin{document}

\title{VeriFlow: Modeling Distributions for Neural Network Verification}

\author{Faried Abu Zaid \inst{1}  \and Daniel Neider  \inst{2}\inst{3} \and Mustafa Yal\c{c}{\i}ner \inst{2}\inst{3}}
\institute{
    Independent Researcher, Munich, Germany \and
    TU Dortmund University, Dortmund, Germany \and
    Center for Trustworthy Data Science and Security, University Alliance Ruhr, Dortmund, Germany 
}

\maketitle

\begin{abstract}
Formal verification has emerged as a promising method to ensure the safety and reliability of neural networks.
However, many relevant properties, such as fairness or global robustness, pertain to the entire input space. If one applies verification techniques naively, the neural network is checked even on inputs that do not occur in the real world and have no meaning.
To tackle this shortcoming, we propose the VeriFlow architecture as a flow-based density model tailored to allow any verification approach to restrict its search to some data distribution of interest.
We argue that our architecture is particularly well suited for this purpose because of two major properties. 
First, we show that the transformation that is defined by our model is piecewise affine. Therefore, the model allows the usage of verifiers based on constraint solving with linear arithmetic.
Second, upper density level sets (UDL) of the data distribution are definable via linear constraints in the latent space. As a consequence, representations of UDLs specified by a given probability are effectively computable in the latent space. This property allows for effective verification with a fine-grained, probabilistically interpretable control of how \mbox{(a-)typical} the inputs subject to verification are. 
\end{abstract}

\section{Introduction}
\newcommand{\target}{\tau}
\newcommand{\Lk}{\ell_k}

The outstanding performance of neural networks in tasks such as object detection \cite{zhao2019object}, image classification \cite{rawat2017deep}, anomaly detection \cite{pang2021deep} and natural language processing \cite{goldberg2016primer} made them a popular solution for many real-world-applications, including safety-critical ones.
With the increasing popularity of neural networks, defects and limitations of these systems have been witnessed by the general public. The AI incident database\footnote{\url{https://incidentdatabase.ai/}} keeps track of harms and near-harms caused by AI-Systems in the real world.

Ideally, safety and fairness properties of such inherently opaque neural networks should be formally guaranteed when used in safety-critical applications.
As a solution, formal verification can be used to check whether a neural network satisfies a given safety property for the entire input space, or whether there exists some (synthetic) input for which the desired property is violated.
This is in contrast to the statistical testing methods classically employed in machine learning, where the output of the neural network is checked for a finite set of samples, usually from a held-out test set.

However, state-of-the-art formal verification methods are designed for verifying either global or local properties. 
Global properties ensure a specific behavior of the neural network on the whole input space. 
As an example, fairness properties require the neural network to predict the same output for any two inputs that only differ in some sensitive attribute.
Local properties, on the other hand, ensure a specific behavior of the neural network only in some part of the input space that is usually restricted using the training set. 
One well-studied example for a local property is \textit{adversarial robustness}, which requires from the neural network for any point from the data set that any minor perturbation of that point does not significantly change the prediction.
However, both global and local properties have shortcomings limiting their applicability.
Local properties suffer from the same problem as statistical testing, i.e., they rely on a high-quality data set that the verification property is based on.
Global properties refer to the entire input space, including regions that we may not need to verify, such as noise-inputs or regions of the input space for which there are only very few training samples available (epistemic uncertainty). 
See Figure~\ref{fig:motivation} for two Marabou-generated counterexamples for a global property on MNIST: one resembles a digit, the other resembles noise --- highlighting the importance of modeling data distributions for verification.

\begin{wrapfigure}{r}{0.5\linewidth}
    \centering

    \begin{subfigure}[b]{0.3\linewidth}
        \includegraphics[width=\linewidth]{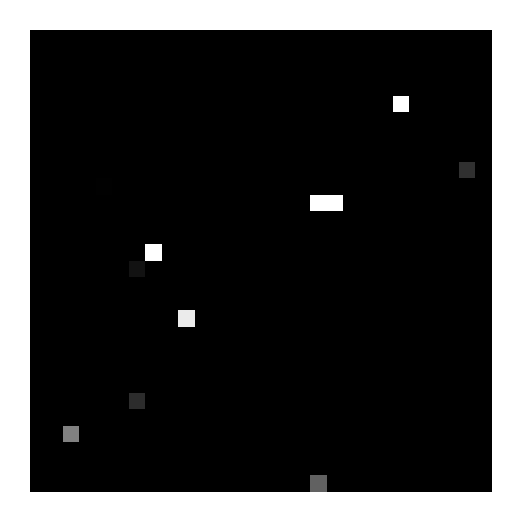}
        \caption*{Without VeriFlow}
        \label{fig:motivation:without-veriflow}
    \end{subfigure}
    \begin{subfigure}[b]{0.3\linewidth}
        \includegraphics[width=\linewidth]{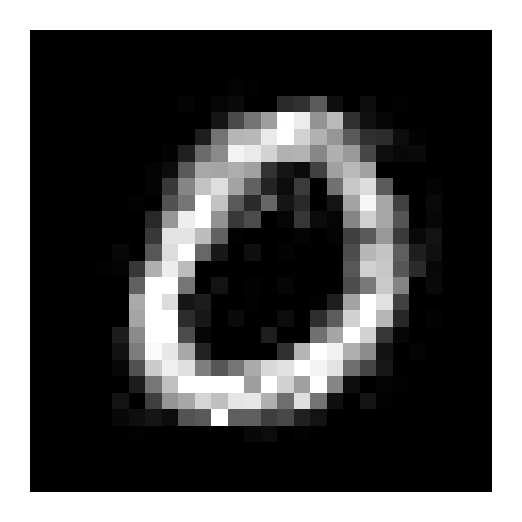}
        \caption*{With VeriFlow}
        \label{fig:motivation:with-veriflow}
    \end{subfigure}
    \caption{Verification of a global property}
    \label{fig:motivation}
\end{wrapfigure}

Specifically, we verify the \textit{global} property: All inputs that are classified as $0$ have low confidence.
We deliberately chose this synthetic property to test the intuition that confident classifications should correspond to typical, recognizable inputs. 
This experiment challenges that assumption.
We verify the property with Marabou twice: once with VeriFlow, and once without.
As a result, we obtain two input assignments as counterexamples that are classified as digit 0 but with high confidence by the ReLU-based neural network.
Confidence is computed as ${\text{conf}(y) := y_i - \left(\sum_{j \neq i} y_j\right)/|y|}$, where $y$ is the logit vector, $i=0$ is the target class, and $|y|=10$ denotes the number of logits.
The input on the left was found without constraining the input space and resembles noise --- yet it is still classified as 0 with high confidence.
The input on the right was found by applying VeriFlow with distributional constraints, yielding a digit that visually resembles a typical 0. 
This illustrates two key insights: First, even for properties that intuitively suggest in-distribution counterexamples, the verifier may return highly atypical inputs. Second, VeriFlow enables the discovery of meaningful, in-distribution counterexamples for neural classifiers.

In order to restrict global properties to some data distribution of interest, generative adversarial networks (GANs) and variational autoencoders (VAEs) can be used to verify the global property only in the high-density region of the data distribution.
However, while these models are indeed able to capture the data distribution, they do not allow for probabilistic interpretability. 
For example, variational autoencoders are generally not designed to ensure that the low-density (high-density) data points from the training data set are mapped to the tail (center) of the latent space distribution. 

To overcome this issue, we design a flow model tailored towards the application in neural network verification.
We leverage our flow model to restrict the input space of the neural network under verification to the underlying data distribution.
This allows restricting the verification of global properties to the data distribution of interest.
In contrast to GANs and VAEs, our VeriFlow architecture does not only allow for efficient sampling but also provides probabilistic interpretability via tractable likelihoods \cite{papamakarios_normalizing_2021,goodfellow2014generative,rezende2015variational,dinh_nice_2015,tabak2010density}.
This feature is unique to our flow architecture and important to facilitate fine-grained probabilistic control when restricting the input space to typical inputs during verification. 
This approach makes the verification property less reliant on the dataset, while still keeping the input space focused on meaningful data.

\paragraph{Summary}
The goal of VeriFlow is to enable restricting the verification of a neural network to a probabilistically meaningful subset of the input space.
To this end, we propose a \textbf{bijective, invertible, and piecewise linear} flow model that \textbf{preserves monotonic relationship between latent and learned densities} by mapping high-density (and low-density) regions in the data distribution to high-density (and low-density) regions in the base distribution.
Crucially, \textbf{high-density regions are defined in terms of $\Lk$-norms in the base distribution}, allowing us to express the high-density data region as a set of linear constraints. 
These constraints, in turn, can be incorporated into the specification of downstream verification tasks.
In the following two sections, we formally introduce all bolded properties and sketch the VeriFlow architecture.
\textbf{In our experiments}, we present a novel notion of global robustness, which requires the neural network to be robust across the entire high-density input region. 
Crucially, the high-density region is precisely characterized by our flow model. 
While our global robustness property is desirable on its own, we demonstrate how our specification can accelerate the verification of local robustness and discuss how the two notions are semantically connected.
Finally, we evaluate VeriFlow's learning ability and show that it outperforms existing uniformly scaling flow architectures, often by a large margin.

\section{Preliminaries}

\textbf{Verifying neural networks} involves checking if a network \( f \) satisfies a semantic property \( P \), often expressed as \( \phi(x) \implies \psi(f(x)) \), where \(\phi\) and \(\psi\) are pre- and postconditions. Two conceptually different verification approaches are relevant for this paper: \textit{constraint-based verification} and \textit{abstract interpretation}. We briefly explain the core idea of each verification approach and refer to~\citet{albarghouthi-book} for an in-depth explanation.
In \textbf{constraint-based verification}, the network \( f \) and  property \( P \) are translated into a logical formula \(\Psi_{f,P}\), whose validity implies that \( f \) satisfies \( P \). 
Verifying \(\Psi_{f,P}\) involves checking the unsatisfiability of \(\neg\Psi_{f,P}\). If \(\neg\Psi_{f,P}\) is satisfiable, a counterexample exists witnessing a violation of $p$ by $f$.
Efficient SMT solvers like Marabou 2~\cite{wu2024marabou} handle these formulas for networks with piecewise linear components such as ReLU activations. Constraint-based verification methods are \textit{complete} and always provide counterexamples when the property is violated, though their runtime can be high for positive proofs.
\textbf{Abstract Interpretation} symbolically executes a neural network using geometric abstract domains, such as zonotopes or star sets, which over-approximate input sets. By propagating these through the network, the procedure over-approximates possible outputs. Verification succeeds if the output lies entirely within the ``safe space'' defined by the semantic property \(\psi\). If outputs are outside, the property fails; partial overlap leaves the result inconclusive due to the over-approximation in the process.

\textbf{Flow-based models} are a class of neural networks designed to learn complex data distributions, supporting both sample generation and density estimation. 
Specifically, these models learn a bijective, continuously differentiable transformation — known as a \emph{diffeomorphism} — that maps a simple base distribution \( B \) with probability density function $p_B$ to a complex data distribution \( D \) with probability density function $p_D$.
In our setting, we define the flow in the direction from the base distribution to the data distribution, that is, from latent to observed variables. 
Let \( F \) denote the learned transformation, implemented as an invertible neural network. The invertibility of \( F \) allows us to leverage the model for two complementary tasks:

\begin{enumerate}
  \item \textbf{Sampling:} Draw a sample \( z \sim B \), then apply the forward map to obtain \( x = F(z) \). The result, $x$, is a new synthetic sample from the learned data distribution.
 
  \item \textbf{Density estimation:} Given a data point \( x \), compute its likelihood under the model using the change-of-variables formula~\cite{folland_real_1999}: ${p_D(x) = \left| \det \frac{\partial F^{-1}}{\partial x^T} \right| \cdot p_B(F^{-1}(x)).}$ 
  \end{enumerate}
If the determinant of the $\det \frac{\partial F^{-1}}{\partial x^T}$ is constant, then we say that the flow is \textbf{uniformly scaling}.  
\subsection{Base Distributions}
The most commonly used base distribution is the \emph{isotropic Gaussian}, a multivariate normal distribution with zero mean and identity covariance matrix. 
When a flow model uses this distribution as its base, it is referred to as a \emph{normalizing flow}.
While the Gaussian distribution is a natural choice, it is by no means the only option. 
In fact, the Knothe-Rosenblatt Rearrangement Theorem~\cite{knothe_contributions_1957,rosenblatt_remarks_1952} guarantees that for any two absolutely continuous probability distributions, there exists a smooth, invertible mapping (a \emph{diffeomorphism}) that transforms one distribution into the other. 
This foundational result justifies using alternative base distributions: the expressive power of flow models is theoretically unaffected by the choice, as long as the flow is flexible enough.

Alternative base distributions relevant to us are those that are \emph{$k$-radial monotonic}, where the probability density depends solely on the $\Lk$-norm of the input:

\begin{definition}[Radial Distributions and Radial Profile]
    Let $k \in \mathbb{N}_{>0} \cup \{\infty\}$, and let $X$ be a random variable in $\mathbb{R}^d$ with density function $p(x)$. We say that $X$ is \emph{$k$-radially distributed} if there exists a function $g : \mathbb{R}_+ \to \mathbb{R}_+$ such that $p(x) = g(|x|_k)$,
    where $|x|_k$ denotes the $\Lk$-norm of $x$. The function $g$ is called the \textbf{radial profile} of $X$. If $g$ is strictly decreasing, then $X$ is called \emph{$k$-radial monotonic}.
\end{definition}

Intuitively, a $k$-radial distribution is one where all points at the same $\Lk$-distance from the origin have the same density. If the density decreases as the norm increases, the distribution is \emph{radially monotonic}. This includes well-known distributions like the Gaussian (for $k=2$) and the Laplace distribution (for $k=1$).
Importantly, $k$-radial distributions are completely determined by  the distribution of the $\Lk$-norm. Here, the norm distribution w.r.t. a random variable $X$ refers to the distribution of $|X|_k$.
The following definition makes this explicit by showing how to construct a $k$-radial distribution from a one-dimensional norm distribution:

\begin{definition}[$k$-Radial Distribution Construction]
    Let $\rho: \mathbb{R}_+ \to \mathbb{R}_+$ be a one-dimensional probability density function, and let $k \in \mathbb{N}_{>0} \cup \{\infty\}$. The profile of the \emph{$k$-radial distribution} in $d$-dimensional space with norm distribution $\rho$ is given by 
    $g(r) = \rho(r) \left(\frac{\partial V^d_k(r)}{\partial r}\right)^{-1}$,
    where $V^d_k(r)$ denotes the volume of the $\Lk$-ball of radius $r$ in $\mathbb{R}^d$.
\end{definition}

Here, the derivative $\frac{\partial V^d_k(r)}{\partial r}$ captures the rate at which the volume increases as we move outward in norm, i.e. $g(r)$ describes the density of the corresponding radial distribution at points x with $|x|_k=r$. The formula ensures that when we integrate the density over spherical shells, i.e., sets of points with the same norm, we recover the desired norm distribution $\rho$.

Although the above construction enables precise density estimation for every point via the norm distribution, density estimators are often used more coarsely --- such as in our work --- to determine whether a point lies in a high- or low-density region.
To formalize this, we define upper and lower density level sets, which represent regions containing high-density samples (center of the distribution) and low-density samples (tail of the distribution), respectively.

\begin{definition}
Given the density of the input distribution $p_D:\mathbb{R}^d \rightarrow \mathbb{R}_+$, the set of points whose density exceeds a given threshold $t$ is called the upper density level set (UDL) and is defined as ${L^{\uparrow}_D(t) \coloneqq \{ x \in \mathbb{R}^d \mid p_D (x) \geq t \}}$.
Conversely, the lower density level set (LDL) contains the set of points subceeding the density threshold: ${{L^{\downarrow}_D (t)\coloneqq\{ x \in \mathbb{R}^d \mid p_D (x) \leq t \}}= \mathbb{R}^d\setminus L^{\uparrow}_D}$. 
If for $q\in[0,1]$ there is a unique UDL of $D$ with probability $q$, then we denote this set by $\text{UDL}_D(q)$.
\end{definition}
Note that the existence and uniqueness of the UDL with a given probability is guaranteed if for all $t>0$ the equality ${P_D(\{x\mid p(x)_D =t\}) = 0}$ holds, where $P_D$ denotes the probability induced by $p_D$, defined as ${P(x\in S) := \int_S p_D(x)dx}$.

Crucial for our application is the fact that density level sets of $k$-radial distributions can be described in terms of quantiles of the radial profile. Clearly, each density-level set is the union of $k$-spherical shells, i.e. $\text{UDL}_X(q) = \{ x \mid |x|_k \in R\}$ for some $R\subseteq \mathbb{R}^+$. Therefore, in order to describe $\text{UDL}_X(q)$, it is sufficient to describe its norm projection $R=\{|x|_k \mid x\in \text{UDL}_X(q)\}$, which has a very simple box shape for radial monotonic distributions:

\begin{restatable}{observation}{rad_udls}\label{prop:rad_udls}
    Let $X$ be a $k$-radially distributed random variable with radial profile $g$. Then 
    $\{|x|_k \mid x\in \text{UDL}_X(q)\} = \{r \mid g(r) > t\},$
    where $t=\textrm{quantile}_{g(|X|_k)}(1-q)$. Moreover, if $X$ is radial monotonic, then 
    $\{|x|_k \mid x\in \text{UDL}_X(q)\} = [0, \text{quantile}_{|X|_k}(q))$.
\end{restatable}

\subsection{Piecewise Affine Flows} 
A function $f: \mathcal{X}\to \mathcal{Y}$ is piecewise affine, if there exists a partition of the domain $\mathcal{X} = \mathcal{X}_1 \cup \cdots \cup \mathcal{X}_n$ such that $f$ restricted to $\mathcal{X}_i$ is affine for every $i$. 
We call $\mathcal{X}_1,\ldots, \mathcal{X}_n$ affine regions of $f$.
As we argued earlier, piecewise affinity is crucial for efficient SMT-based verification. Therefore, our flow transformation should be piecewise affine, which we can achieve e.g., by using ReLU networks. If we can ensure that the defined function is bijective, then we obtain a continuous piecewise affine bijection where
the affine regions can be represented as intersections of open and closed half-spaces \cite{moser_tessellationfiltering_2022}.
Hence, the regions are contained in the Borel algebra $\mathcal{B}(\mathbb{R}^d)$. 

Proposition~\ref{prop:pa-flow} asserts that the change of variables formula remains valid for piecewise affine bijections (instead of diffeomorphisms), provided that the affine regions are Borel sets.

 \begin{restatable}{prop}{paflow}\label{prop:pa-flow}
     Let $F:\mathbb{R}^d\to\mathbb{R}^d$ be a piecewise affine bijection, with affine regions $\mathcal{X}_1,\ldots,\mathcal{X}_n\in\mathcal{B}(\mathbb{R})^d$ and
     let $B$ be an absolutely continuous random variable. Then, 
     $p_{F(B)}(x) = p_{B}(F^{-1}(x))\left|\det\frac{\partial F^{-1}}{\partial x}\right|$,
     where the Jacobian of $F^{-1}$ is evaluated piecewise, according to the affine regions.
 \end{restatable}

\begin{proof}
For $i\in\{1\ldots,n\}$, let $f_i$ be the affine function that coincides with $F$ on $\mathcal{X}_i$.  We consider the conditional probability densities of type
\[p_{B}(x\mid x\in \mathcal{X}_i) = \frac{p_B(x) \cdot \mathbb{I}[x\in \mathcal{X}_i]}{P(B\in \mathcal{X}_i)}.\]
 Since $F$ is an affine bijection, i.e. a diffeomorphism on the  
the support of $p_B(x| x\in \mathcal{X}_i)$, $\mathcal{X}_i$, we can employ the change of variables formula and obtain that 
\begin{align*}
&p_{F(B)}(y\mid F^{-1}(y) \in\mathcal{X})_i = \frac{p_B(f_i^{-1}(y)) \left|\det\frac{\partial f_i^{-1}}{\partial y}\right| \cdot \mathbb{I}[F^{-1}(y)\in \mathcal{X}_i]}{P(B\in \mathcal{X}_i)}.
\end{align*} 
Since $\mathcal{X}_1,\ldots, \mathcal{X}_n$ is given by $P_B(\mathcal{X}_i)\cdot p_{F(B)}(x \mid x\in \mathcal{X}_i)$, we can compute $p_{F(B)}$ by the sum rule to obtain: 
\begin{align*}
    p_{F(B)}(x) &= \sum_{i=1}^n P(x\in\mathcal{X}_i) p_{F(B)}(x\mid x\in\mathcal{X}_i) \\
    &=\sum_{i=1}^n p_B(f_i^{-1}(x)) \left|\det\frac{\partial f_i^{-1}}{\partial x}\right| \cdot \mathbb{I}[F^{-1}(x)\in \mathcal{X}_i] \\
                &\stackrel{\ast}{=} p_B(F^{-1}(x)) \sum_{i=1}^n \left|\det\frac{\partial f_i^{-1}}{\partial x}\right| \cdot \mathbb{I}[F^{-1}(x)\in \mathcal{X}_i] \\
                &= p_{B}(F^{-1}(x))\left|\det\frac{\partial F^{-1}}{\partial x}\right|,
\end{align*}
where (*) is shown in the next Lemma \ref{lem:make_f_great_again}.
\end{proof}

\begin{lemma}\label{lem:make_f_great_again}
Let $F$ be a piece-wise affine bijection with affine regions $\mathcal{X}_1,\ldots,\mathcal{X}_n$. Further, let $f_1,\ldots, f_n$ be the affine functions that coincide with $F$ on $\mathcal{X}_i$, respectively. Then,
    \begin{align*}    
        \sum_{i=1}^n p_B(f_i^{-1}(x)) \left|\det\frac{\partial f_i^{-1}}{\partial x}\right| \cdot \mathbb{I}[F^{-1}(x)\in \mathcal{X}_i] = p_B(F^{-1}(x)) \sum_{i=1}^n \left|\det\frac{\partial f_i^{-1}}{\partial x}\right| \cdot \mathbb{I}[F^{-1}(x)\in \mathcal{X}_i].
    \end{align*}
\end{lemma}

\begin{proof}
    \begin{align*}
        \sum_{i=1}^n p_B(f_i^{-1}(x)) \left|\det\frac{\partial f_i^{-1}}{\partial x}\right| \cdot \mathbb{I}[F^{-1}(x)\in \mathcal{X}_i] &= \sum_{i=1}^n\sum_{j=1}^n p_B(f_j^{-1}(x)) \left|\det\frac{\partial f_i^{-1}}{\partial x}\right| \mathbb{I}[F^{-1}(x)\in \mathcal{X}_j]\delta_{ij} \\
                &\stackrel{\ast}{=} \sum_{i=1}^n \underbrace{\left(\sum_{j=1}^n p_B(f_j^{-1}(x))\mathbb{I}[F^{-1}(x)\in \mathcal{X}_j]\right)}_{=p_B(F^{-1}(x))}  \cdot \left |\det\frac{\partial f_i^{-1}}{\partial x}\right| \cdot \mathbb{I}[F^{-1}(x)\in \mathcal{X}_i] \\
                &= p_B(F^{-1}(x)) \sum_{i=1}^n \left|\det\frac{\partial f_i^{-1}}{\partial x}\right| \cdot \mathbb{I}[F^{-1}(x)\in \mathcal{X}_i] 
    \end{align*}
    where $(\ast)$ holds since 
\[\mathbb{I}[F^{-1}(x)\in R_j]\delta_{ij} = \mathbb{I}[F^{-1}(x)\in R_i]\cdot \mathbb{I}[F^{-1}(x)\in R_j],\] where $\delta_{ij}=\begin{cases} 1; i=j \\ 0; \text{ else}\end{cases}$ is the Kronecker-Delta.
\end{proof}

\section{VeriFlow: A Density-Preserving Architecture} 
In this section, we build upon the preliminary results and present the novel property of VeriFlow: the preservation of density level sets between the data distribution and learned distribution.
This relationship distinguishes uniformly scaling flows from most other flow classes and justifies using the density level sets of the base distribution as a presentation for those of the data distribution.
Specifically, Proposition~\ref{prop:levelsets} shows that uniformly scaling flows preserve a monotonic relationship between latent and learned densities:

\begin{restatable}[Density Level Set Preservation]{prop}{pudl}
  \label{prop:levelsets}
    Let $F$ be a uniformly scaling flow with base distribution $B$. Then $F$ maps upper density level sets of the base distribution to upper density level sets of the data distribution: $\mathrm{UDL}_{F(B)}(q) = F(\mathrm{UDL}_{B}(q))$, where $F(B)$ is the random variable defined by $F$ with base distribution $B$ and $F(\mathrm{UDL}_{B}(q))$ is the image of $\mathrm{UDL}_{B}(q)$ under $F$.
\end{restatable}

\begin{proof}
  This is a direct consequence of the change of variables formula.

\begin{align*}
    F\left( \{ x \mid \log p_D(x) > t \} \right) &= \{ F(x) \mid \log p_D(x) > t \} \\
    & = \left\{ F(x) \;\middle|\; \log p_B(F(x)) > t -
        \underbrace{\log \left| \det \frac{\partial F}{\partial x} \right|}_{\text{const}} \right\} \\
    & = \{ y \mid \log p_B(y) > t' \}
\end{align*}
  which is obviously an upper log-density level set w.r.t. the latent
  distribution $B$. The last equation holds since $F$ is a bijection and $\log \left| \textrm{det}\frac{\partial
  F}{\partial x}  \right|$ is constant. Note that $P_B(S) = P_{F(B)}(F(S))$ does always hold. 

\end{proof}

Combining Proposition~\ref{prop:levelsets} and Observation~\ref{prop:rad_udls} yields that uniformly scaling flows with radial base distributions (w.r.t $\ell_1$ or $\ell_\infty$ norm) allow for simple density level set  descriptions via linear constraints in latent space:

\begin{restatable}{cor}{us_and_rm}
  \label{cor:us_and_rm}
    Let $F$ be a uniformly scaling flow and $B$ a $k$-radial base distribution with radial profile $g$. Then 
    $\{|F^{-1}(x)|_k \mid x\in \text{UDL}_{F(B)}(q)\} = \{r \mid g(r)\geq t\}$
    where $t = \text{quantile}_{g(|B|_k)}(1-q)$. Moreover, if $B$ is radial monotonic, then 
    $\{|F^{-1}(x)|_k \mid x\in \text{UDL}_{F(B)}(q)\} = [0, \text{quantile}_{|B|_k}(q))$.

\end{restatable}
In particular, for a radial base distribution w.r.t. the $1$- or the $\infty$-norm, density level sets of the learned distribution can be defined in latent space using linear constraints.

Next we show that if the density of the base distribution is log-piecewise affine, then so is the log-density of the transformed distribution. Note that we do not need a uniformly scaling flow for the statement to hold.
\begin{restatable}{prop}{padensity}\label{prop:pw_affine}
  Let $p_D$ be defined by a piecewise affine flow $F$ and a log-piecewise affine base distribution $p_B$. Then $\log p_D$ is piecewise affine. \ 
\end{restatable}

\begin{proof}
  As we have seen,
  
  \begin{align*}
    \log p_D (\mathbf{x}) &= \log
    p_{B} (F (x)) + \log \left| \textrm{det}\frac{\partial
  F}{\partial x}  \right|
  \end{align*}
  
  Since $F$ and $\log p_B(\cdot)$ are piecewise affine, $\log p_{B} (F (x))$ is also piecewise affine. Similarly, $\left|\textrm{det}\frac{\partial F}{\partial x} \right|$ is piecewise constant, which implies that $\log\left|\textrm{det}\frac{\partial F}{\partial x} \right|$ is piecewise constant too. The claim follows immediately.
\end{proof}

\subsection{The Components of VeriFlow} To identify learnable building blocks that satisfy our aforementioned requirements, we survey established flow architectures. 
A central finding is that additive conditioning layers with ReLU activations, such as additive coupling, yield precisely the kind of networks we seek. 
These include additive coupling layers~\cite{dinh_nice_2015}, additive auto-regressive layers~\cite{kingma_improved_2016, huang_neural_2018, papamakarios_masked_2017}, and masked additive convolutional flows such as MACow~\cite{ma_macow_2019}.
Furthermore, bijective affine transformations parameterized by LU-decomposed matrices, named LUNets~\cite{chan_lunet_2023}, are also suitable. 
LU layers constitute a powerful replacement for the permutations or masks that were originally proposed to be combined with coupling-like layers. Although this addition turned out to be very beneficial, we note that the number of parameters scales quadratically with the dimension, which is a bottleneck in high dimensions. The poor scalability of vanilla LU-layers led us to propose one-star-convolutions --- convolutions with kernel size 1 along all spatial dimensions, adapted to input topology (e.g., 1D for sequences, 2D for images).
We apply one-star-convolutions together with LU-decomposed bijective channel transforms, providing a natural way to define a LU-decomposed bijective affine transformation on the entire set of input components via parameter sharing~\cite{kingma2018glow}. Based on these layers, we first propose the VeriFlow architecture, and then provide an in-depth analysis of the above mentioned layers in the remainder of this section.

The VeriFlow architecture is built from blocks of shape $B_i = A_i^{-1}\circ C_i \circ A_i$, where $A_i$ is an affine transform given by a one-star-convolution with $LU$ decomposed bijective channel transform and $C$ is an additive coupling layer. This shape is motivated by the above analysis and results from the normalizing flow literature showing that applying an adjoint group action with respect to the subgroup of affine transforms of the group of diffeomorphisms over $\mathbb{R}^d$ to a coupling layer is more beneficial than simply alternating coupling and affine transforms~\cite{hoogeboom2019emerging, ma2019flow++}.
The complete flow is of shape
$
{F = A_{n+1} \circ B_n \circ B_{n-1} \circ \cdots \circ B_1,}
$
where the final affine transform ensures that the flow can have an arbitrary (constant) Jacobian determinant. Note that each block is volume preserving since determinants of 
$A_i$ and $A_i^{-1}$ cancel each other, and additive coupling layers are volume preserving by design~\cite{dinh_nice_2015}.
In conclusion, we obtain the following results for our architecture.

\begin{restatable}[VeriFlow is Uniformly Scaling]{observation}{veriflow-us}
\end{restatable}
\begin{restatable}[VeriFlow is Piecewise Affine]{observation}{veriflow-pl}\label{prop:pa-veriflow}
\end{restatable}
Next, we present an in-depth analysis of all building blocks of VeriFlow.

\subsubsection{Additive Coupling (NICE)} \label{sec:additive_coupling}
As it turns out, additive transformations yield precisely the properties that we need in order to guarantee the good properties of the previous section.
The simplest such architecture is realized by so-called additive coupling, which was first introduced for
the NICE architecture by~\cite{dinh_nice_2015}.
NICE belongs to the first flow architectures. Nevertheless, it is still a popular benchmark which has
shown good performance on multiple data sets.
A NICE flow is built from additive coupling layers. Each such layer $L$
consists of a partition $I_1, I_2$ of $[D]$, where $D$ is the data dimension,
and a conditioning function $m : \mathbb{R}^d \rightarrow \mathbb{R}^{D - d}$,
where $d = | I_1 |$. The layer $L$ maps $x$ to $y$ where
\begin{eqnarray*}
  y_{I_1} & = & x_{I_1}\\
  y_{I_2} & = & x_{I_2} + m (x_{I_1}) .
\end{eqnarray*}

It is easy to see that the Jacobian $\frac{\partial y}{\partial x} =
\left(\begin{array}{cc}
  I_d & 0\\
  \frac{\partial y_{I_2}}{\partial x_{I_1}} & I_{D-d}
\end{array}\right)$ is triangular and that all entries on the diagonal
are $1$. As the first $d$ components of the input remain unchanged, it is
usually necessary to employ multiple layers with varying partitions of the
input vector. It is straightforward to see that a coupling layer defines a piecewise affine function if the conditioner $m$ is piecewise affine.

\paragraph{Allowing Rescaling}

As all additive coupling layers have Jacobian determinant $1$, the same will
hold for their composition. That means the space is never stretched or compressed through the transformation, which potentially limits the expressiveness. In order to account for this issue, NICE allows
for a final  component-wise rescaling, i.e. multiplication with a matrix
$S$, where $S_{{ij}} \begin{cases}
  \neq 0 & \text{if } i = j\\
  = 0 & \text{else}
\end{cases}$. \

\paragraph{Computing log-Densities}

Because of the simple additive coupling, computing log-densities is
particularly simple. Let $F$ be a NICE flow with base distribution $B$, layers
$L_1, \ldots, L_n$, and scaling matrix S. Then
\begin{align*}
  \log (p_D (\mathbf{x})) & = \log \left( p_B (F (x)) \left|
  \textrm{det}\frac{\partial F}{\partial x} \right| \right) \nonumber\\
  & = \log \left( p_B (F (x)) \cdot \prod \left| \textrm{det}\frac{\partial
  L_i}{\partial x} \right| \cdot |\det S | \right) \nonumber\\
  & =  \log p_B (F (x)) + \underbrace{\sum \log \left| \textrm{det}\frac{\partial
  L_i}{\partial x}  \right|}_{= 0} + \log |\det S | \nonumber\\
  & = \log p_B (F (x)) + \sum \log S_{ii} 
\end{align*}

Computing $F^{- 1} (z)$ has exactly the same complexity as computing a forward
pass $F (x)$. Because in order to invert the flow we only need to multiply
with the inverse scaling matrix and then pass the input through the inverse
coupling layer in reverse order. Note that for an additive coupling layer $L =
((I_1, I_2), m)$ the inverse function can be implemented by $L^{- 1} = ((I_1,
I_2), - m)$.

\subsubsection{Masked Additive Coupling}
It is also possible to rewrite the additive coupling equation in order to implement the NICE architecture as a fully connected neural network with masking and skip connections. An additive 
coupling layer $\ell: \genfrac(){0pt}{2}{x_{I_1}}{x_{I_2}} \mapsto  \genfrac(){0pt}{2}{x_{I_1}}{x_{I_2}+c(x_{I_1})}$, whose conditioner is implemented by a neural network $c$ can equivalently be written as 
\begin{align}
\ell(x) = x + (1-\textrm{mask})\cdot c'(\textrm{mask}\cdot x),\label{eqn:mask}
\end{align}
where $\textrm{mask}$
is a $\{0, 1\}$-vector with $\textrm{mask}_i = 1 \Leftrightarrow i\in I_1$ and the multiplication is computed component-wise. Further, $c'$ is a fully connected network obtained by adding dummy inputs for the components in $I_2$ and dummy outputs for the components of $I_1$ to $m$, which are effectively eliminated by the mask in Equation \ref{eqn:mask}. 

\subsubsection{Additive Auto-Regression}
A general way to increase the expressiveness of the  based flow models is the use of auto-regression instead of coupling~\citep{kingma_improved_2016,huang_neural_2018,papamakarios_masked_2017}. In this case the conditioner is implemented as an RNN $c$, which couples the input component by component. More precisely, an additive auto-regressive flow layer $\ell$ computes a transformation $\ell(x) = y$ with
\begin{align*}
    &h_1, z_1 = 0;\hspace{1em} (h_{i + 1},z_{i+1}) = c(x_i, h_{i}) \\
    &y_i = x_i + z_i 
\end{align*}
Observe that the structure of the auto-regression still leads to a lower triangular shape of the Jacobian and the additive auto-regressive coupling ensures that all diagonal entries are $1$. With these properties, one easily checks Proposition \ref{prop:pa-flow}, \ref{prop:pw_affine} and \ref{prop:levelsets} remain valid if additive coupling is replaced by additive auto-regression.

\subsubsection{Masked Additive Convolutions}
The idea of masking was used by~\cite{ma_macow_2019} in order to transfer coupling to convolutional architectures where the input is a higher-order tensor. We can also employ this idea in our situation and still maintain the desired properties. In this case, Equation \ref{eqn:mask} is applied to a convolutional network, e.g. with a checker board and/or a channel-wise mask. As the reader readily verifies, the analogues of Proposition \ref{prop:pa-flow} - \ref{prop:levelsets} hold also for this layer.

\subsubsection{LUNets}
Recently, bijective fully-connected layers have been proposed by~\cite{chan_lunet_2023} as a so-called LUNet. The idea is to ensure that both, the affine transformation of a fully connected layer and the non-linearity are bijections. Bijectivity is ensured by representing the linear transform of the layer by an LU-factorization $A = LU$ with lower/upper triangular Matrices $L$ and $U$. Bijectivity is ensured by adding the constraints that the diagonal of $L$ contains only ones and the diagonal of $U$ is always non-zero. In this case, Propositions \ref{prop:pa-flow} and \ref{prop:pw_affine} will still hold if we replace the layer architecture and use leaky ReLU instead of ReLU, but Proposition \ref{prop:levelsets} will in general not hold anymore as the determinant of the layer Jacobian is not constant anymore. 

LUNet is a very different approach to guaranteeing the bijectivity of the transformation compared to additive coupling. It has the advantage that the entire input can be transformed by a single layer.
The restriction that the affine transform needs to be bijective, however, fixes the capacity of the transformation to $d^2$ parameters where $d$ is the input dimension. This can be problematic, especially when working with high-dimensional data. 

\subsubsection{Bijective Affine Layers}
The bijective affine transform $T(x) = (LU)x + b$ at the heart of an LU-layer deserves special attention. Note that the determinant of the Jacobian is constant for $T$. Since computing the inverse of an affine transform also has the complexity of computing the affine transform, it follows that we can add bijective affine layers to the list of benign layers. Bijective affine layers can be an interesting alternative to the intermediate permutation layers of the NICE architecture. Using an affine bijection instead of a simple fixed permutation of the dimensions allows the architecture to correlate the components of $I_1$ and $I_2$ in the subsequent coupling layer in a learnable fashion. As an example, consider the extreme case where all components of the target distribution are independent. In this case, the components $I_1$ and $I_2$ will be independent, no matter which permutation of the components we have applied beforehand. An affine bijection, however, can be capable of combining the variables in a way such that the components $I_1$ and $I_2$ become correlated.

\subsection{The Neuro-Symbolic Verification Pipeline}
The key idea of the \emph{neuro-symbolic verification} framework as proposed by ~\citet{DBLP:conf/ijcai/XieKN22} is to use neural networks as part of the specification to enable the verification of a variety of complex properties.
Our approach is aligned with the Neuro-Symbolic verification framework as we enable the verification of neural network properties over the UDLs or LDLs of proxy distributions learned by VeriFlow. 
The UDL in latent space captures likely inputs under the base distribution, while the UDL in the target (data) space captures typical inputs resembling those in the dataset. 
Verifying properties within a UDL thus proves/falsifies that the property holds for all \emph{likely} inputs.

The verification pipeline proceeds as follows.
First, we define a density level set in the latent space using linear constraints (Corollary~\ref{cor:us_and_rm}). 
These constraints are transformed through the flow model into output constraints that precisely characterize the corresponding upper density level set in the data space, when a constraint-based verification tool is employed (Proposition~\ref{prop:levelsets}). 
This is enabled by the piecewise linearity of the flow model (Proposition~\ref{prop:pa-flow}). 
In contrast, abstract interpretation uses the input bounds to compute an over-approximation of the output region, yielding a sound but conservative approximation of the same density level set.
We illustrate this pipeline using an abstract interpretation-based verifier in Figure~\ref{fig:abstact-interpretation}.

\begin{figure*}[t]
  \centering

\tikzset{every picture/.style={line width=0.75pt}} 
\resizebox{0.9\textwidth}{!}{%
\begin{tikzpicture}[x=0.75pt,y=0.75pt,yscale=-1,xscale=1]

\draw   (317.78,87.93) -- (392.54,109.46) -- (392.71,138.73) -- (318.2,161.11) -- cycle ;
\draw    (110.02,51.5) -- (151.12,52.07)(109.98,54.5) -- (151.08,55.07) ;
\draw [shift={(160.1,53.7)}, rotate = 180.8] [fill={rgb, 255:red, 0; green, 0; blue, 0 }  ][line width=0.08]  [draw opacity=0] (10.72,-5.15) -- (0,0) -- (10.72,5.15) -- (7.12,0) -- cycle    ;
\draw[color={rgb, 255:red, 184; green, 233; blue, 134 }] (47.86,85.25) -- (87.81,129.25) -- (43.82,169.19) -- (3.87,125.2) -- cycle ;
\draw   (102,105.07) .. controls (102,107.7) and (109.32,109.84) .. (118.35,109.84) .. controls (127.39,109.84) and (134.71,107.7) .. (134.71,105.07) .. controls (134.71,102.44) and (142.03,100.3) .. (151.06,100.3) .. controls (160.09,100.3) and (167.42,102.44) .. (167.42,105.07) -- (167.42,143.23) .. controls (167.42,140.6) and (160.09,138.46) .. (151.06,138.46) .. controls (142.03,138.46) and (134.71,140.6) .. (134.71,143.23) .. controls (134.71,145.86) and (127.39,148) .. (118.35,148) .. controls (109.32,148) and (102,145.86) .. (102,143.23) -- cycle ;
\draw    (320.41,53) -- (360.79,52.75)(320.43,56) -- (360.81,55.75) ;
\draw [shift={(369.8,54.2)}, rotate = 179.65] [fill={rgb, 255:red, 0; green, 0; blue, 0 }  ][line width=0.08]  [draw opacity=0] (10.72,-5.15) -- (0,0) -- (10.72,5.15) -- (7.12,0) -- cycle    ;
\draw  [fill={rgb, 255:red, 184; green, 233; blue, 134 }  ,fill opacity=1 ] (261.42,69.5) -- (299.32,124.67) -- (222.42,177.5) -- (184.52,122.33) -- cycle ;
\draw  [fill={rgb, 255:red, 255; green, 255; blue, 255 }  ,fill opacity=1 ] (203.96,130.14) .. controls (200.79,125.64) and (200.42,119.75) .. (202.98,114.97) .. controls (205.55,110.2) and (210.62,107.36) .. (216.03,107.68) .. controls (216.02,104.11) and (217.69,100.78) .. (220.51,98.69) .. controls (223.34,96.6) and (227,96) .. (230.38,97.08) .. controls (230.3,93.88) and (231.85,90.9) .. (234.49,89.18) .. controls (237.12,87.47) and (240.47,87.27) .. (243.34,88.65) .. controls (244.18,84.78) and (247.17,81.78) .. (251.01,80.94) .. controls (254.84,80.11) and (258.84,81.6) .. (261.28,84.77) .. controls (264.43,84.1) and (267.74,84.85) .. (270.35,86.83) .. controls (272.97,88.82) and (274.62,91.83) .. (274.89,95.1) .. controls (279.49,97.29) and (282.62,101.77) .. (283.14,106.85) .. controls (283.65,111.93) and (281.46,116.87) .. (277.39,119.81) .. controls (279.41,123.78) and (279.18,128.51) .. (276.77,132.18) .. controls (274.37,135.85) and (270.17,137.89) .. (265.79,137.51) .. controls (266.86,142.97) and (264.73,148.5) .. (260.33,151.72) .. controls (255.93,154.94) and (250.05,155.26) .. (245.22,152.55) .. controls (242.99,156.29) and (239.34,158.93) .. (235.11,159.87) .. controls (230.88,160.82) and (226.42,159.99) .. (222.73,157.57) .. controls (219.58,159.73) and (215.4,159.64) .. (212.27,157.36) .. controls (209.14,155.08) and (207.73,151.09) .. (208.73,147.37) .. controls (204.79,147.13) and (201.4,144.35) .. (200.32,140.48) .. controls (199.24,136.61) and (200.72,132.53) .. (203.99,130.36) ; \draw   (208.73,147.37) .. controls (210.59,147.48) and (212.43,147.01) .. (213.99,146.03)(222.73,157.57) .. controls (223.39,157.12) and (223.99,156.59) .. (224.51,155.98)(245.22,152.55) .. controls (244.33,152.05) and (243.5,151.45) .. (242.73,150.78)(265.79,137.51) .. controls (265.6,136.51) and (265.3,135.54) .. (264.91,134.6)(277.38,119.81) .. controls (275.23,115.59) and (270.88,112.95) .. (266.21,113.04)(274.89,95.1) .. controls (275.04,96.85) and (274.78,98.6) .. (274.15,100.21)(261.28,84.77) .. controls (261.68,85.29) and (262.03,85.85) .. (262.33,86.45)(243.34,88.65) .. controls (243.13,89.62) and (243.06,90.61) .. (243.13,91.6)(230.38,97.08) .. controls (230.4,97.85) and (230.52,98.61) .. (230.72,99.35)(216.03,107.68) .. controls (217.17,107.75) and (218.31,107.96) .. (219.4,108.3)(203.96,130.14) .. controls (204.39,130.76) and (204.88,131.34) .. (205.4,131.89) ;
\draw    (416.42,84.5) -- (436.42,84.5) ;
\draw    (426.42,84.5) -- (426.05,163.15) ;
\draw    (416.23,163.7) -- (435.23,163.7) ;
\draw  [fill={rgb, 255:red, 184; green, 233; blue, 134 }  ,fill opacity=1 ] (421.02,93.7) -- (431.25,93.7) -- (431.25,140) -- (421.02,140) -- cycle ;
\draw  [fill={rgb, 255:red, 255; green, 255; blue, 255 }  ,fill opacity=1 ] (421.01,101.25) -- (431.26,101.25) -- (431.26,132.45) -- (421.01,132.45) -- cycle ;
\draw   (484,13) -- (658.27,13) -- (658.27,83.3) -- (484,83.3) -- cycle ;
\draw   (492.7,23) -- (512,23) -- (512,42.3) -- (492.7,42.3) -- cycle ;
\draw  [color={rgb, 255:red, 0; green, 0; blue, 0 }  ,draw opacity=1 ][fill={rgb, 255:red, 184; green, 233; blue, 134 }  ,fill opacity=1 ] (492.7,54) -- (512,54) -- (512,73.3) -- (492.7,73.3) -- cycle ;
\draw    (439.35,102.02) -- (448.13,102.27) ;
\draw    (448.13,102.27) -- (448.25,132.88) ;
\draw    (439.25,132.88) -- (448.25,132.88) ;
\draw    (448.19,117.57) -- (456.33,117.6) ;

\draw (15,42) node [anchor=north west][inner sep=0.75pt]   [align=left] {Latent UDL};
\draw (115.1,120.2) node [anchor=north west][inner sep=0.75pt]   [align=left] {Flow};
\draw (203,41) node [anchor=north west][inner sep=0.75pt]   [align=left] {Data UDL};
\draw (208.29,114.18) node [anchor=north west][inner sep=0.75pt]  [rotate=-1.33] [align=left] {$\displaystyle UDL_{D}( q)$};
\draw (20,116) node [anchor=north west][inner sep=0.75pt]   [align=left] {$\displaystyle \mathbb{B}_{p}^{d} (r(q))$};
\draw (324,116) node [anchor=north west][inner sep=0.75pt]   [align=left] {Classifier};
\draw (397.95,33) node [anchor=north west][inner sep=0.75pt]   [align=left] {Classifier\\Range};
\draw (443,74) node [anchor=north west][inner sep=0.75pt]   [align=left] {$\displaystyle 1$};
\draw (439,154) node [anchor=north west][inner sep=0.75pt]   [align=left] {$\displaystyle 0$};
\draw (522,27) node [anchor=north west][inner sep=0.75pt]   [align=left] {Exact Set};
\draw (522.4,56) node [anchor=north west][inner sep=0.75pt]   [align=left] {Over-approximation};
\draw (457.6,109) node [anchor=north west][inner sep=0.75pt]   [align=left] {$\displaystyle C( UDL_{D}( q))$};
\end{tikzpicture}
} 
\caption{Visualization of a verification pipeline for an in-distribution verification of a classifier using VeriFlow. The procedure starts by defining the UDL exactly in the latent space. The true classification range w.r.t. the UDL equals the result of pushing the set consecutively through the flow and the classifier. An over-approximation can be obtained via abstract interpretation.}\label{fig:abstact-interpretation}
\end{figure*}
To ensure their reliability in formal verification settings, we carefully evaluate and validate the trained flow models.
\subsubsection{Validation and Calibration}
Capturing the true input distribution by the flow as closely as possible is crucial for meaningful verification results. 
Therefore, assessing the deviation between true and learned distribution empirically is a natural part of the verification pipeline. 
For normalizing flows, typical assessment methods exploit that the maximum likelihood objective of the flow training is equivalent to minimizing the KL divergence between the empirical and the defined base distribution. 
One can therefore measure discrepancies in the latent space~\cite{DBLP:journals/jmlr/PapamakariosNRM21, DBLP:conf/nips/LinhartG023}.
Here, we use the term empirical base distribution to denote the pre-image of the true data distribution under the flow.

Since we work with radial base distributions and are focused on density level sets, we test for discrepancies in norm distribution and radiality separately. More precisely, we propose the following measures to assess the quality of the empirical latent norm distribution and the defined norm distribution in the latent space:
\begin{enumerate}[label={(\roman*)}]
\item KS Statistic which quantifies the maximum difference between two cumulative distribution functions: $\sup_{r\in\mathbb{R}^+} |\text{CDF}_{|F^{-1}(X)|}(r) - \text{CDF}_{|B|}(r)|$~\cite{Kolmogorov1933, Smirnov1939},
\item and kernel density estimates (KDEs) and probability–probability (PP) plots using a consistent non-parametric estimator. KDEs visualize the density of the transformed and base distributions, while PP-plots compare their empirical quantiles.
\end{enumerate}
These measures give us quantitative information about the maximal and the local deviation between the learned and the true distribution, measured in probability. 
Since we project into the one-dimensional norm-space, we can obtain reliable estimates in a sample-efficient way. The tests regarding the norm-distribution are complemented by two radiality tests: \begin{enumerate}[label={(\roman*)}]
\item We assess sign symmetry using a binomial test --- that is, whether positive and negative signs in the latent dimensions appear with equal frequency, as would be expected in a truly symmetric radial distribution, 
\item We apply an energy distance test to the normalized absolute values of the latent vectors to assess whether these directions are uniformly distributed over the simplex (projection to the positive part of $\ell_1$ sphere), another key property of a radial distribution~\cite{binomial_test_2008, rizzo_energy_2016}. 
\end{enumerate}
All tests provide $p$-values, which we can combine into a single rejection criterion by applying Bonferroni correction~\cite{Dunn1961}. 
The separate assessment of norm distribution and radiality allows us to distinguish different types of discrepancies. For instance, a well-aligned empirical norm-distribution but with discrepancies in radiality indicates that the flow over-approximated the true distribution.

In case that moderate discrepancies in the norm distribution have been identified, e.g. $0.05>p>0.0001$, we may use the norm distribution of the empirical base distribution for recalibration: 
 We adjust the thresholds for the base distribution such that the returned density level set in the target distribution contains the desired fraction of the data. 
 
\subsubsection{Interpretability}
Recalibration of density level sets allows to maintain interpretability, even in the presence of discrepancies between the true and the learned distribution: Calibration against a test set $D_\text{Cal}$ ensures that a density level set with target probability $q$ will indeed contain a $q$-fraction of $D_\text{Cal}$. Hence, a successful verification will always establish a stronger statement than the empirical check $P_{x\sim D_\text{Cal}}(\phi(x)) \geq q$, since an infinite number of variants are verified along with the contained data.  

\section{Experiments}

We implement the VeriFlow as an open-source Python library that can be trained on all datasets mentioned in this paper in an out-of-the-box fashion. Our library ensures that definable flows guarantee all the theoretical properties outlined in this paper. The library implements general radial distributions with learnable norm distribution for various norms. 
The resulting flow models can be exported in ONNX format, relying exclusively on ONNX node types that are supported by the VNN-LIB standard~\cite{DBLP:conf/cav/DemarchiGPT23}. Additionally, we provide evaluation methods for the previously outlined quality assessment of the flow.

\subsection{Evaluation \& Validation} \label{sec:eval-validate}
A crucial step before using VeriFlow in verification is to empirically check whether the data distribution was learned faithfully.
To this end, we showcase our validation procedure on flows that were trained on MNIST digits, which are of particular interest for the verification experiments. As proposed earlier, we use KS Statistic~\cite{Kolmogorov1933, Smirnov1939} and PP-Plots for validation. 
Results are shown in Figure~\ref{fig:flow-eval} for all MNIST digits and in Figure~\ref{fig:flow-eval:fashion} for all FashionMNIST classes. 
Both validation methods show that the latent norm distribution of the data matches the defined latent norm distribution almost perfectly. 

For the KS statistic, we provide an example with the flow trained on digit $0$ (because we use digit $0$ for the subsequent verification experiment) and confirm that the KS value of $0.0375$ ($p=0.152 \gg 0.05$) gives no statistically significant evidence for discrepancies. The other digits yield similar results, though not all directly meet the acceptance criterion.
The radiality test, however, fails. Despite an average $p$-value of $0.0778$, only about half of the dimensions pass the Bonferroni corrected sign symmetry test ($p>0.05/(2\cdot 784)$).
Similarly, the projected uniformity test fails with $p$ being numerically $0$. 
We conclude that the learned UDLs contain the corresponding true data UDLs very well but are also likely to contain some OOD data (over approximation). 
Evaluations for the remaining digits as well as samples from the learned distribution are contained in the Appendix.
These validations of VeriFlow’s faithfulness, together with the theoretical analysis presented in the previous section, support its use in the verification domain. 
In the following section, we demonstrate how VeriFlow can be leveraged to enhance the verification of a global property.

\begin{figure*}[htbp]
    \centering
    \begin{subfigure}[b]{0.48\textwidth}
        \centering
        \resizebox{\linewidth}{!}{%
            \input{pp_plot_combined-mixure.pgf}
        }
        \caption{MNIST digits}
        \label{fig:flow-eval}
    \end{subfigure}
    \hfill
    \begin{subfigure}[b]{0.48\textwidth}
        \centering
        \includegraphics[width=\linewidth]{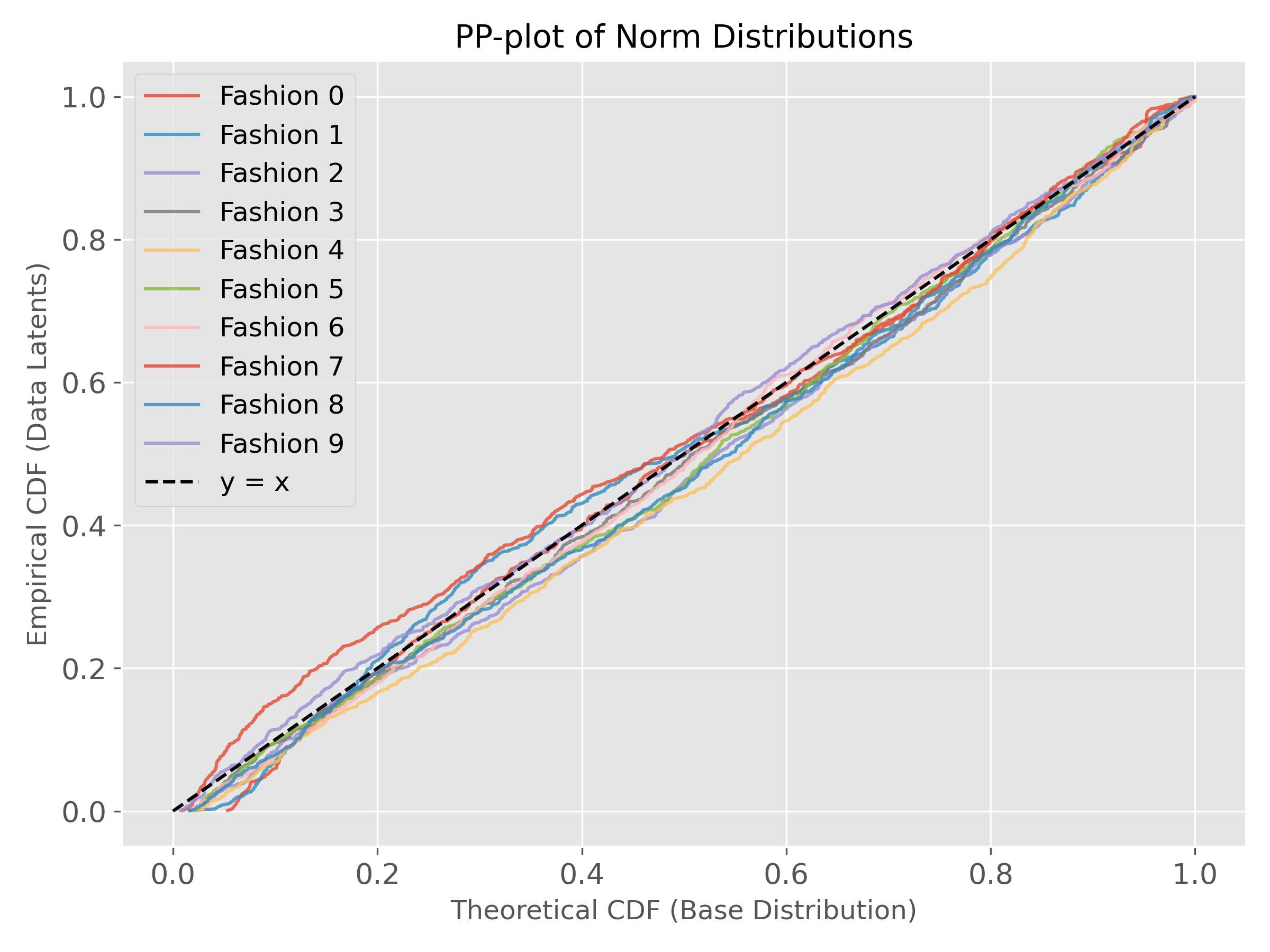}
        \caption{FashionMNIST classes}
        \label{fig:flow-eval:fashion}
    \end{subfigure}
    \caption{Comparison of PP-plots for VeriFlow on MNIST and FashionMNIST datasets. In both cases, the flexibility of the learnable norm distribution allows the model to tightly align norm distributions and the latent norms of the data.}
    \label{fig:flow-eval:comparison}
\end{figure*}

\newcommand{\classifier}{c}
\newcommand{\flow}{F}
\newcommand{\dataset}{\mathcal{D}}
\newcommand{\dist}{\mathrm{dist}}
\newcommand{\density}{q}
\subsection{Robustness Verification} 

We demonstrate one use case of VeriFlow in the context of local robustness --- a popular benchmark specification~\cite{DBLP:journals/corr/abs-2412-19985}.
Local robustness involves checking whether a neural network's prediction remains unchanged for minor perturbations.
More formally, let \( \classifier: \mathbb{R}^n \to \mathbb{R}^m \) be a classifier, let \( \dataset \subseteq \mathbb{R}^n \) be a dataset. The network is \emph{locally robust} for a point $x\in \dataset$ w.r.t. a distance metric \( \dist \) and radius \( \epsilon > 0 \), if the following specification is satisfied:
\begin{gather*}
    \phi\colon\bigl\{x,x'  \in\mathbb{R}^n \bigr\} \\
    \mathrel{ y \gets \classifier(x) \hspace{6pt} y' \gets \classifier(x')} \nonumber\\
    \psi\colon\bigl\{ \dist(x,x')\leq \epsilon \Rightarrow \arg\max_{i} y = \arg\max_{i} y' \bigr\}\nonumber
\end{gather*}
This~\citet{DBLP:journals/cacm/Hoare69} triple notation \(\{\phi\}\ \texttt{assignments}\ \{\psi\}\) corresponds to the logical form \(\phi(x) \Rightarrow \psi(f(x))\) from the background section.
Clearly, the runtime for verifying local robustness over the entire dataset increases linearly with its size $|\dataset|$. 

To tackle this scalability problem, we demonstrate how to leverage our flow model to verify local robustness on multiple instances with only one verification run.
To formalize this, let $\flow:\mathbb{R}^n \to \mathbb{R}^n$ be a flow model with base distribution $B$ and $\classifier: \mathbb{R}^n \to \mathbb{R}^m $ be a classifier.
Then, we propose verifying the following global robustness specification:
\begin{gather*}
    \phi\colon\bigl\{x_\ell \in UDL_B(\density), x'  \in\mathbb{R}^n \bigr\} \\
    \mathrel{ x_t \gets \flow(x_\ell) \hspace{6pt} y \gets \classifier(x_t)} \hspace{6pt} \mathrel{y' \gets \classifier(x')} \\
    \psi\colon\bigl\{ \dist(x_t,x')\leq \epsilon \Rightarrow \arg\max_{i} y = \arg\max_{i} y' \bigr\}
\end{gather*}
where $\density \in [0,1]$ is a probability density threshold. By the design of our flow model, the data contained in the UDL, $\{\flow(x_\ell) \mid x_\ell\in UDL_B(\density)\}$, includes $1-\density$ fraction of the dataset $D$.
Intuitively, our specification is a new notion of \textit{global robustness} as our formulation becomes independent of the dataset $\dataset$, but still keeps the input space restricted to meaningful inputs aligned with the data distribution $\dataset$.
In other words, we verify robustness for all likely inputs. 
Therefore, a certificate produced for our global robustness specification also holds for the instance-wise local robustness specification for a $1-\density$ fraction of the dataset $\dataset$.
Note that the inverse direction does not apply: a robustness certificate for some fraction of the dataset $\dataset$ does not imply global robustness as proposed in our specification due to the additional synthetic data that may be contained in our flow model.

\begin{figure}[t]
    \centering
    \begin{subfigure}[b]{0.55\textwidth}
        \centering
        \resizebox{\linewidth}{!}{%
            \input{runtime-combined-overlay.pgf}
        }
        \caption{Runtimes of robustness verification and falsification.}
        \label{fig:robust-runtime-cmp}
    \end{subfigure}
    \hfill
    \begin{subfigure}[b]{0.3\textwidth}
        \centering
        \resizebox{\linewidth}{!}{%
            \includegraphics{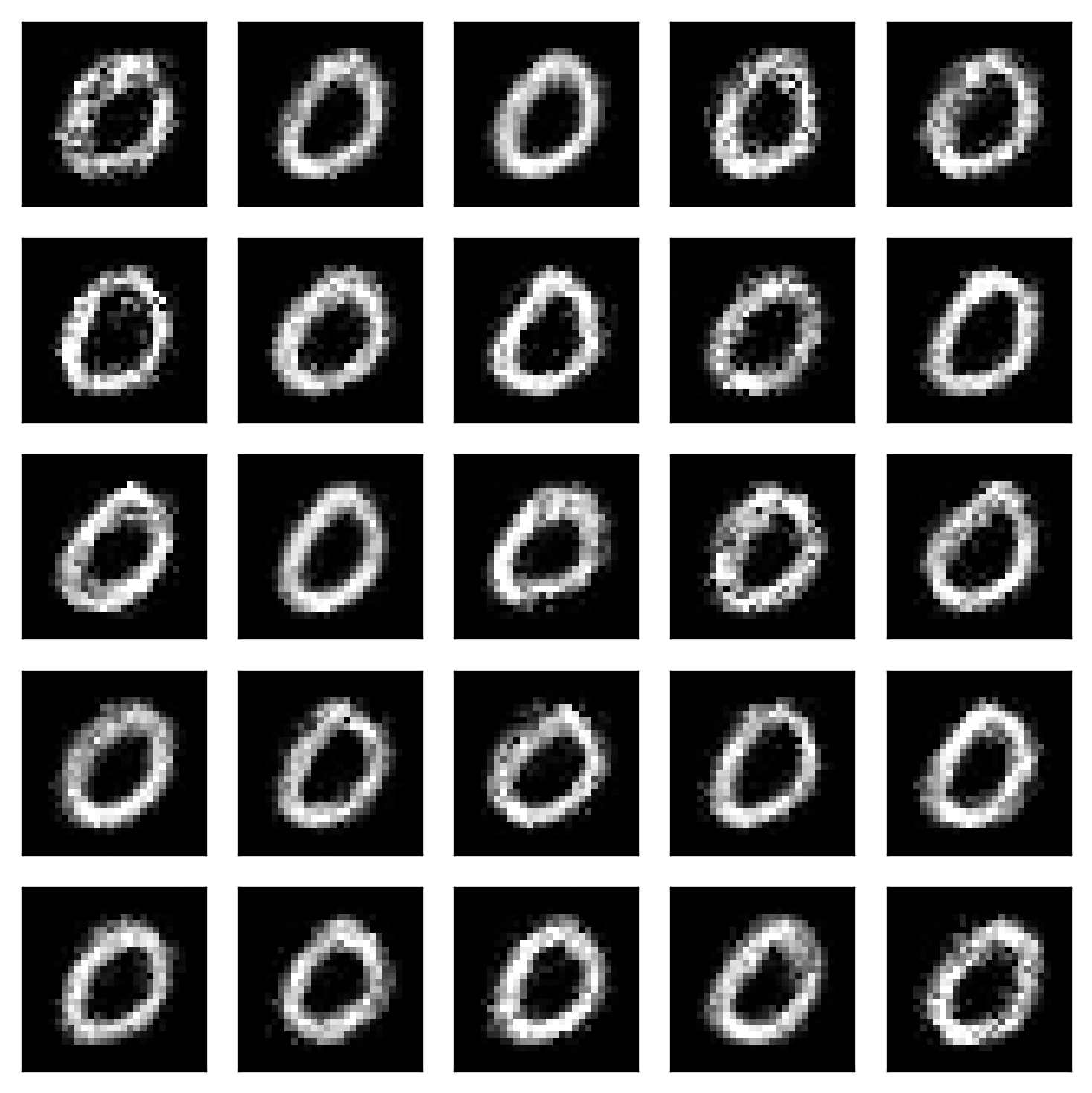}
        }
        \caption{Random samples from the flow model used in the verification experiments.}
        \label{fig:flow-quality-verification}
    \end{subfigure}
    \caption{Runtime verification results and model samples.}
\end{figure}

In order to practically evaluate how the runtime of verifying global robustness compares to verifying local robustness, we conduct experiments on the MNIST dataset. 
Specifically, we aim to analyze after how many instance-wise verifications the overhead of integrating a flow model into the specification is paid off.
To this end, we train our VeriFlow model \( \flow: \mathbb{R}^{28 \times 28} \to \mathbb{R}^{28 \times 28} \) on digit-zero samples from the MNIST dataset (See Figure~\ref{fig:flow-quality-verification} for samples). 
Our convolutional flow model has one coupling block with the same adjoint coupling architecture as presented in the main section and uses a $1$-radial lognormal base distribution. Training is performed using the Adam optimizer. For classification, we employ a two-layer ReLU network defined as \( c: \mathbb{R}^{28 \times 28} \to \mathbb{R}^{10} \).
We employ the SOTA verification tools Marabou 2~\cite{DBLP:conf/cav/WuIZTDKRAJBHLWZKKB24}, and the ABCrown framework's alpha-crown implementation~\cite{DBLP:conf/iclr/XuZ0WJLH21} with branch-and-bound for complete verification.
Specifically, we use (only) Marabou 2 to verify our global robustness specification since ABCrown does not support the type of specification proposed here~\cite{DBLP:journals/corr/abs-2506-12084}.
For verification, we use both Marabou 2 and ABCrown on a Dell Latitude 7440 Laptop with an i7 processor, 32 GB RAM, and Linux 22.04.
In all our verification experiments, we first train VeriFlow in a preceding step and export it to the ONNX format, using only node types that are supported by both Marabou and ABCrown.
Since our global robustness specification requires reasoning jointly about the flow model and the classifier within a single MILP instance, we merge both models into a single ONNX graph. This step is necessary because most verification tools, including Marabou, do not support loading multiple separate ONNX files.
To encode global robustness, we insert two identical copies of the classifier into the ONNX file --- one for the original input and one for the perturbed input --- consistent with our formal definition of global robustness. The combined ONNX model is then parsed by Marabou for verification.
Parsing this composite model required minor modifications to Marabou's source code, due to inconsistencies of Marabou with the ONNX specification~\citep{onnx}. In particular, we made the following adjustments locally and reported them:
\begin{enumerate}[label=(\roman*)]
    \item \textbf{Scalar multiplication (\texttt{Mul})}: Marabou's parser incorrectly assumed that scalar constants always appear in a specific operand position, an assumption not enforced by ONNX. We removed this assumption to allow for arbitrary operand ordering.
    \item \textbf{\texttt{Flatten} operator}: The parser made incorrect assumptions about parameter ordering, and we observed deviations in the flattened output compared to the official ONNX runtime. We fixed both issues in the parser to ensure correct behavior.
    \item \textbf{Broadcasting in \texttt{Mul} and \texttt{Sub} operations}: Marabou lacked support for automatic shape broadcasting, which is part of the ONNX specification. We extended the parser to handle basic broadcasting cases that occur in our flow model.
\end{enumerate}
We also added custom constraints to the ONNX graph to link the output of the flow model to the input of the two classifiers, with a small perturbation applied to one of them in order to encode the robustness condition precisely.
These modifications ensure a correct translation of the ONNX model into the internal MILP representation. 
However, despite these parser fixes, we observed significant deviations between the outputs produced by Marabou and those computed via the ONNX runtime, and reported them to the Marabou developers. 
As a workaround, we adapt our experimental setup to avoid configurations that trigger the aforementioned issue, while still enabling meaningful runtime analysis, which we describe in the next paragraph.

In the specification of our global robustness verification experiment, instead of restricting the input $x$ to the exact UDL, the values of $x$ are restricted only to a small box centered within the high-density region of the distribution. 
To keep all experiments consistent, we sample from the aforementioned box of the flow model as input for the instance-wise local robustness verification (instead of picking points from the dataset $\dataset$ as presented earlier). 
By doing so, we ensure that the robustness certificate of the flow-based global robustness specification still implies instance-wise robustness.
On a high level, we conduct two variations of our global robustness verification experiment. One where the perturbation is high ($\epsilon=0.1$) for measuring the runtime for \textit{falsification} when the network is not robust, and one where the perturbation is low ($\epsilon=0.001$) for \textit{verification} when the network is indeed robust for all inputs.
We visualize both experiments in the same plot in Figure~\ref{fig:robust-runtime-cmp}.
Our results in Figure~\ref{fig:robust-runtime-cmp} demonstrate after how many instance-wise verifications the additional overhead introduced by the flow model becomes worthwhile.
Specifically, Figure~\ref{fig:robust-runtime-cmp} plots the cumulative runtime of instance-wise verification (falsification) runs to the runtime of our global robustness verification (falsification).
There, we can observe that after around $170$ samples for Marabou and $160$ samples for ABCrown, the runtime for instance-wise verification exceeds the runtime of our global robustness verification of $3.7$ seconds.
For falsification, we can observe that after around $23$ samples with Marabou and $7$ samples with ABCrown, the cumulative runtime of local robustness falsification exceeds the runtime of our global robustness falsification.
Typically, one would expect verification to be computationally more expensive than falsification. 
We conjecture that this is due to the bound propagation method used in Marabou which simplifies, or even trivializes the verification condition in our experiments. 
Since the MNIST test set contains around $1000$ instances for each digit, a (successful) verification of our global robustness specification could already pay off after $20\%$ of the test cases.

\subsection{Ablation Study}
In this section, we show the effectiveness of our architecture as a density estimator and generative model. 
Specifically, we compare VeriFlow against MACow~\cite{ma_macow_2019}, a well-established flow-based model built exclusively from masked additive convolutional coupling layers.
The key difference between these two architectures is that VeriFlow extends MACow by applying an adjoint affine group action to the coupling layers. 
Furthermore, VeriFlow uses a custom radial base distribution with a mixture of gammas norm-distribution, while MACow uses the classical standard normal. 
Otherwise, the parametrization is kept the same, except for a learning rate of $10^{-3}$ and a patience of $3$ for VeriFlow and a learning rate of $10^{-5}$ a patience of $10$ for MACow --- this was necessitated because the training of MACow was unstable otherwise. We use $5$ coupling blocks (except for CIFAR-10, where we use $10$ blocks) and $3$ layers per conditioner with a uniform kernel size of $3$ for both. We use the SophiaG optimizer~\cite{sophia2023}.
Following the description given by~\cite{chan_lunet_2023}, we regularize the parameters of the LU layers. Without any form regularization, we observed exploding determinants on many tasks when working with LU layers.

\begin{wrapfigure}{r}{0.45\linewidth}
    \centering
            \begin{tabular}{c S[table-format=-4.4] S[table-format=-4.4]}
            \toprule
            \textbf{Dataset} & \textbf{MACow} & \textbf{Veriflow (Ours)}\\
            \midrule
            MNIST digit 0 & -1719.908 & \textbf{-2320.757} \\
            MNIST digit 1 & -1577.637 & \textbf{-3123.905} \\
            MNIST digit 2 & -1026.944 & \textbf{-2176.636} \\
            MNIST digit 3 & -1237.165 & \textbf{-2531.264} \\
            MNIST digit 4 & -1033.265 & \textbf{-2254.180} \\
            MNIST digit 5 & -1651.304 & \textbf{-3013.184} \\
            MNIST digit 6 & -1331.267 & \textbf{-2820.397} \\
            MNIST digit 7 &  -891.180 & \textbf{-2238.973} \\
            MNIST digit 8 & -1694.938 & \textbf{-2971.620} \\
            MNIST digit 9 & -2006.115 & \textbf{-3062.877} \\
            \midrule
            Full MNIST & -1250.627 & \textbf{-2389.460} \\
            Fashion MNIST & -847.522 & \textbf{-1386.659} \\
            CIFAR-10 & -4215.694 & \textbf{-4500.373} \\
            \bottomrule
            \end{tabular}
        \caption{Negative log-likelihood (NLL) comparison of MACow and VeriFlow across datasets}
            \label{tab:ablation_results}
\end{wrapfigure}

We train our models on the individual MNIST digits, on the full MNIST, FashionMNIST, and CIFAR-10 datasets on an IdeaPad Laptop with a Ryzen7 7735hs chip and 16GB of RAM. 
As usual for density estimation on (8 bit) images, we used (uniform) dequantization to obtain a continuous distribution, and report the NLL of dequantized images~\cite{DBLP:journals/access/TsubotaA23}. 
We summarize our results in Table~\ref{tab:ablation_results} where we can observe more stable training and considerably higher likelihoods with our architecture.

These quantitative results are also visible in the quality of the random samples from each flow model: Figures~\ref{fig:digi_samples_veriflow} and~\ref{fig:digi_samples_macow} for MNIST, and Figures~\ref{fig:fashion_samples_veriflow} and~\ref{fig:fashion_samples_macow} for FashionMNIST.
For the CIFAR-10 dataset, both models were unable to provide good samples with the architectures that we explored (Figure~\ref{fig:cifar_samples_macow}).

We also compare VeriFlow and MACow in terms of how \textit{faithfully} they capture the data distribution as previously explained in Section~\ref{sec:eval-validate}.
Figures~\ref{fig:flow-eval} and~\ref{fig:flow-eval:fashion} already demonstrated VeriFlow’s ability to align latent norms with the true distribution.
For comparison, Figures~\ref{fig:macow-eval:mnist} and~\ref{fig:macow-eval:fashion} present results for MACow.
The PP-plots reveal that MACow fails to learn the data distribution faithfully.
To complement these evaluations, we provide KS Statistics of all models trained on MNIST digits in Table~\ref{fig:eval:KS-benchmark}.

\begin{figure*}
    \centering
    \begin{subfigure}[b]{0.4\textwidth}
        \centering
        \includegraphics[width=\linewidth]{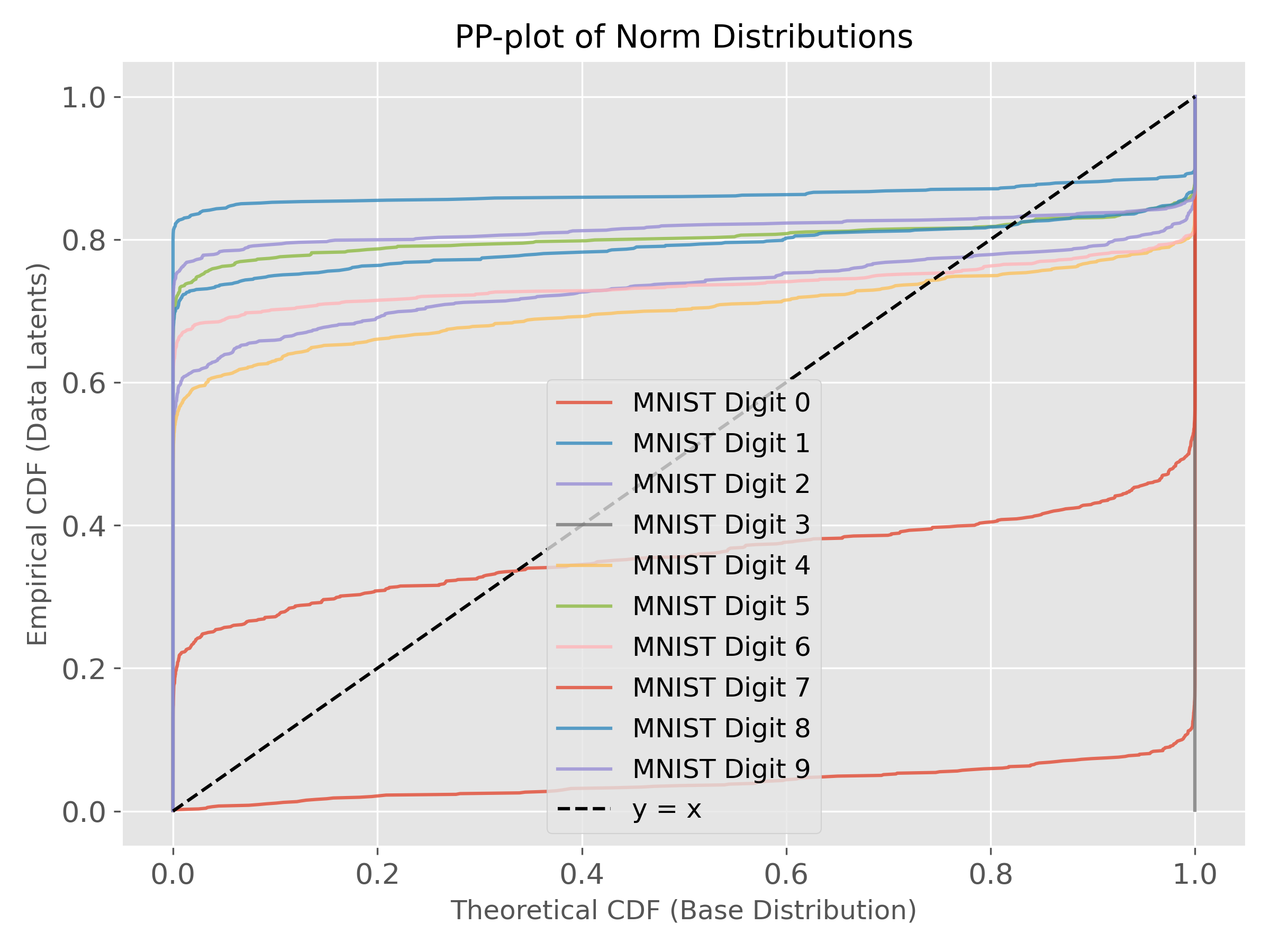}
        \caption{MNIST digits.}
        \label{fig:macow-eval:mnist}
    \end{subfigure}
    \hfill
    \begin{subfigure}[b]{0.4\textwidth}
        \centering
        \includegraphics[width=\linewidth]{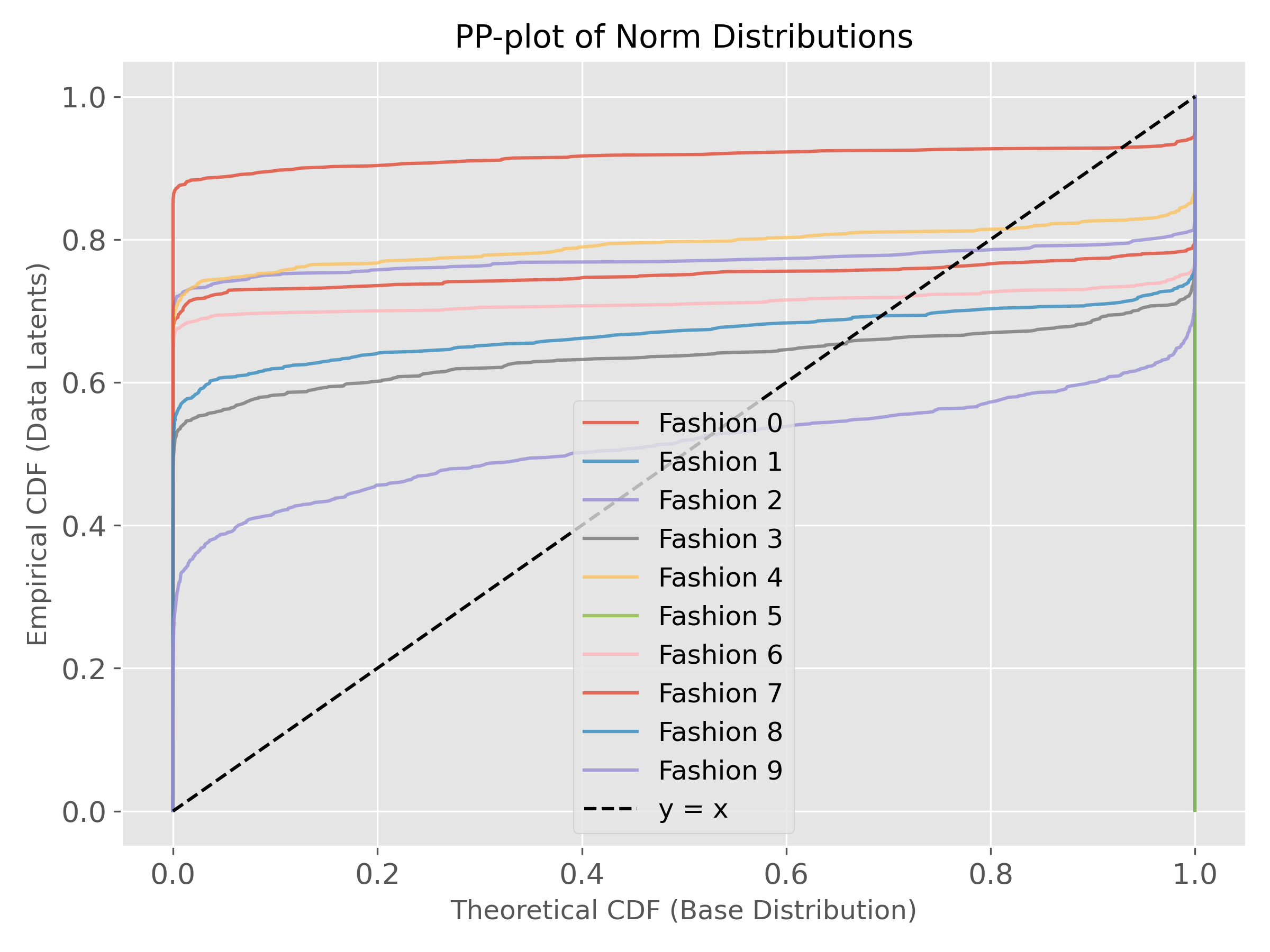}
        \caption{FashionMNIST classes}
        \label{fig:macow-eval:fashion}
    \end{subfigure}
    \caption{Comparison of PP-plots for MACow baseline models on MNIST and FashionMNIST datasets.}
    \label{fig:macow-eval:comparison}
\end{figure*}

\begin{figure*}[b]
    \centering
            \begin{tabular}{c S[table-format=-4.4] S[table-format=-4.4] S[table-format=-4.4] S[table-format=-4.4]}
            \toprule
            \textbf{Dataset} & \textbf{MACow KS} & \textbf{MACow $p$} & \textbf{Veriflow (Ours) KS}
            & \textbf{Veriflow (Ours) $p$}\\
            \midrule
            MNIST digit 0 & 0.887 & \approx 0.0 &  \textbf{0.037} & \textbf{0.152}\\
            MNIST digit 1 & 0.824 & \approx 0.0 & \textbf{0.130} & \textbf{$\approx$ 0.0} \\
            MNIST digit 2 & 0.599 & \approx 0.0 & \textbf{0.066} & \textbf{0.0006}\\
            MNIST digit 3 & 1.0 & \approx 0.0 & \textbf{0.059} & \textbf{0.003} \\
            MNIST digit 4 & 0.575 & \approx 0.0 & \textbf{0.099} & \textbf{$\approx$ 0.0}\\
            MNIST digit 5 & 0.731 & \approx 0.0 & \textbf{0.044} & \textbf{0.079}\\
            MNIST digit 6 & 0.658 & \approx 0.0 & \textbf{0.082} &  \textbf{$\approx$ 0.0}\\
            MNIST digit 7 &  0.505 & \approx 0.0 & \textbf{0.045} & \textbf{0.048}\\
            MNIST digit 8 & 0.706 & \approx 0.0 & \textbf{0.082} & \textbf{$\approx$ 0.0}\\
            MNIST digit 9 & 0.753 & \approx 0.0 & \textbf{0.063} & \textbf{0.001}\\
            \bottomrule
            \end{tabular}
        \caption{KS Statistics of all models trained on MNIST digits. A lower KS value indicates closer agreement between the learned and true distribution, while a higher $p$-value indicates that the difference is not statistically significant.
        All $\approx$0.0 entries denote values smaller $10^{-5}$. We omit to report the radiality tests, since no model was able to pass it.}
        \label{fig:eval:KS-benchmark}
\end{figure*}

\section{Related Work \& Conclusion}
\textbf{Flow models} are on the forefront of modern density estimation and have received significant attention over the last decade~\cite{papamakarios_normalizing_2019}. A constant Jacobian determinant is usually observed at the time of the introduction of the respective layer in the context of the likelihood computation 
~\cite{dinh_nice_2015, ma_macow_2019, kingma_improved_2016,huang_neural_2018,papamakarios_masked_2017}, although typically without further investigation of the induced properties. The role of the Jacobian determinant in general has been investigated in the context of the exploding determinant phenomenon~\cite{kim_softflow_2020,DBLP:journals/corr/abs-2102-06539,lyu_paracflows_2022}. Deeper investigations of uniformly scaling flows have been conducted in~\hbox{\cite{maziarka_oneflow_2022}}, where connections to DeepSVDDs are established by training flows with a volume minimization objective and, recently, in~\cite{Draxler_universality_2024}, which shows that learnable norm-distributions are necessary for the universality of uniformly scaling flows. The latter result confirms observations that we made during our experiments and that led us to implement learnable norm-distributions via mixture models.

Using \textbf{neural networks as part of the specification} to capture semantic properties --- such as meaningful perturbations for improving robustness verification --- has been explored in prior work~\cite{DBLP:journals/corr/abs-2004-14756,DBLP:conf/cvpr/MohapatraWC0D20,DBLP:conf/satml/0001TRYMPHPB23}.
The \emph{neuro-symbolic verification} framework generalizes this idea by proposing a specification language that supports the verification of a neural network with properties specified by other neural networks.~\cite{DBLP:conf/ijcai/XieKN22}.

Most \textbf{verification infrastructure} supports the neural architecture of our flow model and are able to precisely represent the shape of its UDL.
Specifically, Polytopes~\cite{DBLP:conf/aplas/ChenMC08}, Zonotopes, or even simple boxes are sufficient for representing the UDL of our flow model in the latent space precisely. 
This enables the use of a variety of domains such as Deepzono~\cite{DBLP:conf/nips/SinghGMPV18}, DeepPoly~\cite{DBLP:journals/pacmpl/SinghGPV19}, GPUPoly~\cite{DBLP:conf/iclr/SinghGPV19}, RefineZono~\cite{DBLP:conf/iclr/SinghGPV19}, and DiffPoly~\cite{10.1145/3656377}. 
Besides Marabou and ABCrown, all other verifiers that support the VNN-Lib standard~\cite{DBLP:conf/cav/DemarchiGPT23} are conceivable for verification with VeriFlow. 
We refer to the VNN Competition~\cite{brix2023fourth} for an overview of verifiers.

In the past, \textbf{global robustness} specifications have been proposed by~\citet{DBLP:conf/icml/LeinoWF21,DBLP:conf/cav/KatzBDJK17,DBLP:conf/ccs/Chen0QLJW21,10.1007/978-3-030-01090-4_1,DBLP:journals/pacmpl/KabahaD24} along with evaluation methods by~\cite{DBLP:conf/ijcai/RuanWSHKK19,DBLP:conf/nips/ZhangJH022}. However, while some works restrict the input space in their global robustness specification, none of these restrictions are based on the notion of probabilistic interpretability.
\paragraph{Conclusion}
We have presented VeriFlow, a flow-based density model that enables effective verification of neural networks within a learned proxy distribution and fits in existing verification infrastructure.

\begin{figure*}
    \centering
    \includegraphics[width=0.95\linewidth]{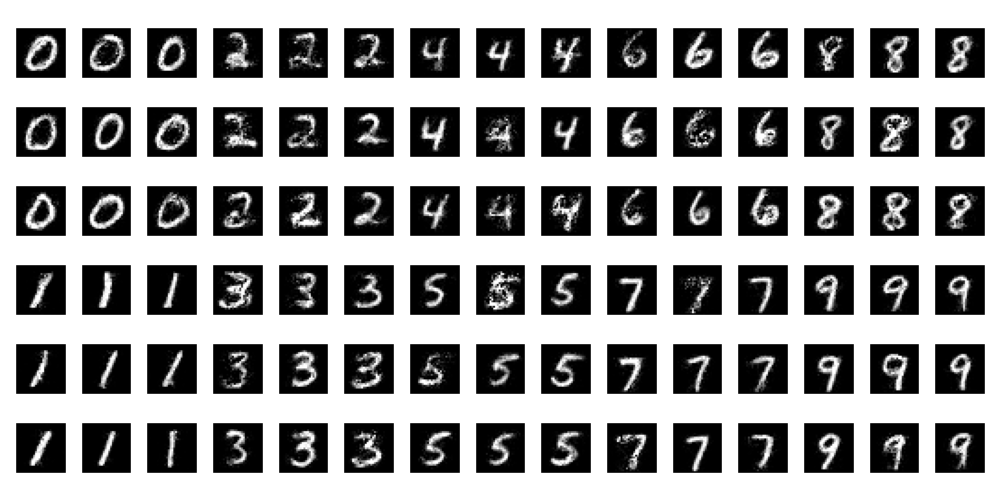}
    \caption{VeriFlow MNIST digit samples.}
    \label{fig:digi_samples_veriflow}
\end{figure*}

\begin{figure*}
    \centering
    \includegraphics[width=0.95\linewidth]{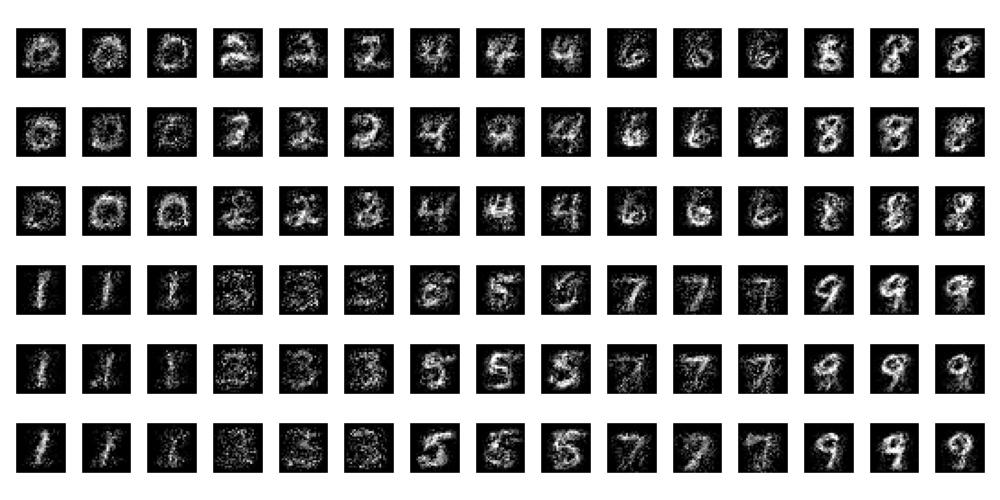}
    \caption{MACow MNIST digit samples.}
    \label{fig:digi_samples_macow}
\end{figure*}

\begin{figure*}
    \centering
    \includegraphics[width=0.95\linewidth]{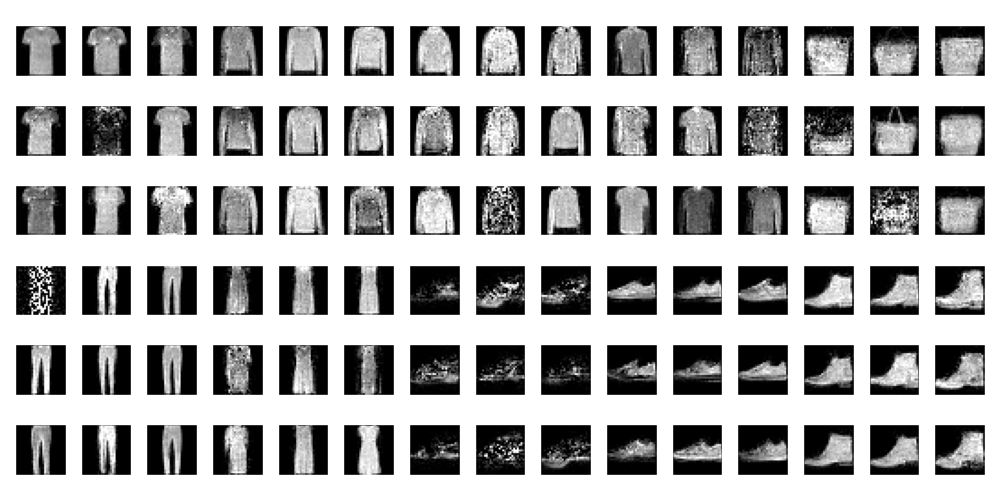}
    \caption{VeriFlow Samples from Fashion classes.}
    \label{fig:fashion_samples_veriflow}
\end{figure*}

\begin{figure*}
    \centering
    \includegraphics[width=0.95\linewidth]{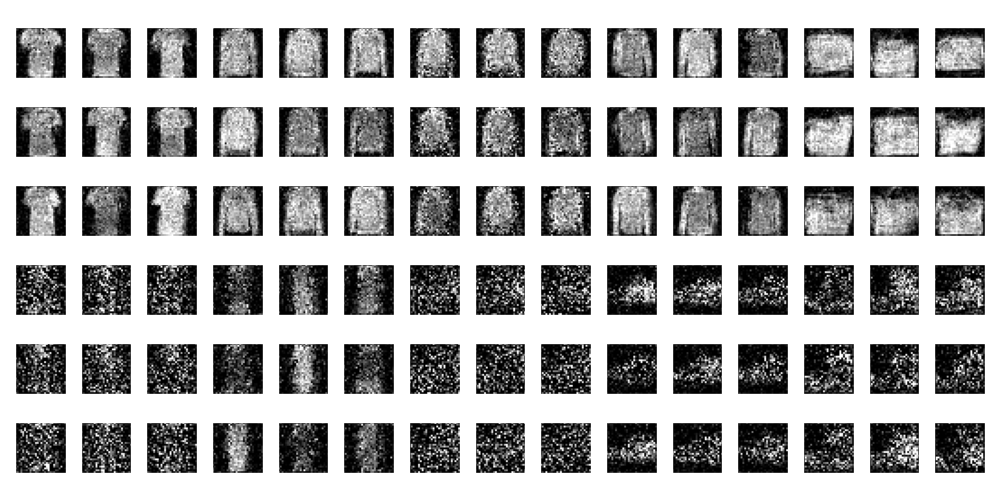}
    \caption{MACow Samples from Fashion classes.}
    \label{fig:fashion_samples_macow}
\end{figure*}

\begin{figure*}
    \centering
    \includegraphics[width=0.95\linewidth]{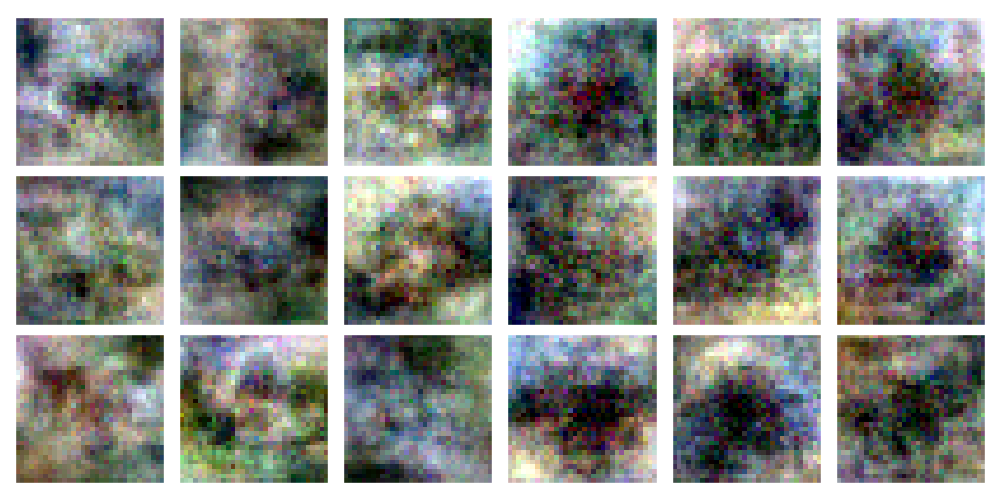}
    \caption{Samples from a VeriFlow model with radial base distribution and log-normal norm distribution (learnable distribution parameters) and from a baseline model (MACow) with standard normal base distribution (no learnable distribution parameters) trained on CIFAR10. The left most 3 columns are samples from the VeriFlow model, while the remaining 3 are samples from the baseline. Both models use 10 coupling blocks with 3 layers per coupling. We assume that a larger model is necessary to obtain better samples.}
    \label{fig:cifar_samples_macow}
\end{figure*}

\section*{Acknowledgements}
This work has been financially supported by Deutsche Forschungsgemeinschaft, DFG Project number 459419731, and the Research Center Trustworthy Data Science and Security (\mbox{\url{https://rc-trust.ai}}), one of the Research Alliance centers within the UA Ruhr (\mbox{\url{https://uaruhr.de}}).
The development of the VeriFlow library (now USFlows, \mbox{\url{https://github.com/aai-institute/USFlows}}) has been financially supported by the Bavarian AI Act Accelerator.

\end{document}